\documentclass{svproc}
\usepackage{multirow}%
\usepackage{mathrsfs}%
\usepackage[title]{appendix}%
\usepackage{manyfoot}%
\usepackage{booktabs}%
\usepackage{listings}%
\usepackage{url}

\usepackage{indentfirst}
\usepackage{mathtools}
\usepackage{upgreek}
\usepackage{tikz}
\usepackage{verbatim}
\usepackage{subcaption}
\usepackage{pgfplots}
\usepackage{epstopdf}
\usepackage{varwidth}
\usepackage{hyperref} 
\usepackage{tabularx}
\usepackage{colortbl}

\usepackage{bm}

\newcommand\fat[1]{\ThisStyle{\ooalign{%
  \kern.46pt$\SavedStyle#1$\cr\kern.33pt$\SavedStyle#1$\cr%
  \kern.2pt$\SavedStyle#1$\cr$\SavedStyle#1$}}}
\definecolor{amber}{rgb}{1.0, 0.75, 0.0}

\usepackage{algorithm}
\usepackage{algorithmicx}
\usepackage{algcompatible}
\usepackage{algpseudocode}
\usepackage{varwidth}

\numberwithin{probs}{section} 
\newtheorem{prob}{\textbf{Problem}}

\usepackage{cite}
\usepackage{amsmath,amssymb,amsfonts}
\usepackage{graphicx}
\usepackage{textcomp}
\usepackage{xcolor}
\def\BibTeX{{\rm B\kern-.05em{\sc i\kern-.025em b}\kern-.08em
    T\kern-.1667em\lower.7ex\hbox{E}\kern-.125emX}}

\begin{document}
\mainmatter              
\title{Learning Backbones: Sparsifying Graphs through Zero Forcing for Effective Graph-Based Learning}
\titlerunning{Learning Backbones}  
%
\author{Obaid Ullah Ahmad\inst{1}
\and Anwar Said\inst{2} \and 
Mudassir Shabbir\inst{3} \and Xenofon Koutsoukos\inst{2} \and Waseem Abbas\inst{1}}
\authorrunning{O. Ahmad et al.} 

\institute{University of Texas at Dallas, Richardson, TX, USA\\
\email{ObaidUllah.Ahmad@utdallas.edu},\\ 
\and
Vanderbilt University,
Nashville, TN, USA\\
\and
Information Technology University,
Lahore, Pakistan\\}

\maketitle              

\begin{abstract}
This paper introduces a novel framework for graph sparsification that preserves the essential learning attributes of original graphs, improving computational efficiency and reducing complexity in learning algorithms. We refer to these sparse graphs as ``learning backbones.'' Our approach leverages the zero-forcing (ZF) phenomenon, a dynamic process on graphs with applications in network control. The key idea is to generate a tree from the original graph that retains critical dynamical properties. By correlating these properties with learning attributes, we construct effective learning backbones. We evaluate the performance of our ZF-based backbones in graph classification tasks across eight datasets and six baseline models. The results demonstrate that our method outperforms existing techniques. Additionally, we explore extensions using node distance metrics to further enhance the framework's utility.

\end{abstract}

\section{Introduction}
\label{sec:intro}

In recent decades, networks have become essential for analyzing complex systems with applications in computer vision~\cite{rital2005weighted}, 3D object modeling~\cite{hamidi2019blind}, and chemical molecules~\cite{morris2020tudataset}. In machine learning, constructing efficient graph representations is critical for tasks like social network analysis, financial systems, and recommendation systems~\cite{wu2020comprehensive}. The complexity of real-world graphs often requires extracting sparse yet informative substructures, known as graph learning backbones, to enable effective learning~\cite{yassin2023modular}. This paper addresses the challenge of identifying these sparse representations while retaining essential properties for downstream tasks by integrating principles from Network Control Theory~\cite{gu2015controllability}.

Control theory, renowned for analyzing and steering dynamic systems~\cite{sontag2013mathematical}, helps select minimal edge sets that capture a graph's intrinsic behavior~\cite{posfai2013effect, barabasi2019twenty}. Viewing a graph as a dynamic system, this approach ensures controllability by maintaining key structural properties.

In graph-based learning, preserving dynamic properties is crucial for accurate classification and prediction~\cite{said2023network}. Previous methods such as graph sparsifiers and spanners~\cite{karger1994using, kruskal1956shortest} aimed to reduce graph complexity while retaining key properties, but often lacked alignment with specific graph learning objectives~\cite{yassin2023modular}. Techniques like spectral rewiring~\cite{chan2016optimizing}, Forman curvature-based rewiring~\cite{topping2021understanding}, and graph diffusion~\cite{gasteiger2019diffusion} have also been explored for optimizing graph structures. However, these methods can sometimes lead to information loss or significant graph densification, which complicates learning tasks. Sparse subgraph extraction techniques have focused on community preservation~\cite{yassin2023modular} but remain limited in balancing both sparsity and learning effectiveness. 

Inspired by tree-like substructures in communication networks~\cite{yu2013connected}, we propose that a connected tree subgraph represents the minimal structure required for learning. These sparse trees, derived from control theory, offer efficient representations while preserving critical properties.

Graph learning has advanced significantly~\cite{said2023network}, but determining the ideal graph structure for specific learning objectives remains a challenge~\cite{tsitsulin2023graph}. This motivates exploration of the Graph Lottery Ticket Hypothesis (GLTH), which posits that within any complex graph, there exists a sparse substructure capable of achieving comparable performance to the full graph~\cite{tsitsulin2023graph}, opening new avenues for scalable graph learning~\cite{chen2021unified}.

However, prevailing methods for uncovering these winning tickets often rely on pruning or sampling, risking information loss~\cite{chen2021unified}. This paper introduces a novel control-theoretic approach to discovering these tickets. The Graph Lottery Ticket Hypothesis is articulated as follows:

\vspace{2 mm} \noindent\fbox{\begin{varwidth}{\dimexpr\linewidth-1\fboxsep-2\fboxrule\relax} \emph{Graph Lottery Ticket Hypothesis}~\cite{tsitsulin2023graph}: For any given graph, there exists a sparse subset of edges such that training any graph learning algorithm solely on this subset yields performance comparable to that of the original graph. \end{varwidth}} \vspace{2 mm}

We present a detailed exploration of our control-theoretic approach to discovering Graph Lottery Tickets (GLTs), which we refer to as learning backbones. We propose that the zero-forcing set (ZFS)-based control backbone~\cite{ahmad2023controllability}, a tree, represents the winning ticket. Our method demonstrates superior precision and efficiency in identifying substructures compared to existing techniques. Additionally, we extend this concept by preserving other control properties in the graph. Through experiments on diverse datasets and tasks, we showcase the exceptional performance and sparsity of the winning tickets identified by our approach.

The rest of the paper is organized as follows: Section \ref{sec:prelim_and_problem} introduces important notations and formulates the main problem of graph sparsification for graph classification. Section \ref{sec:NCB} defines the concept of the ZFS-based backbone and proposes several approaches to compute the learning backbone using control properties of networks. Section \ref{sec:results} presents empirical results for graph classification. Finally, Section \ref{sec:conclusion} concludes the paper and discusses future directions.

In the next section, we review some notations to be used in the rest of the paper and explain the main problem.

\section{Preliminaries and Problem Formulation}
\label{sec:prelim_and_problem}

In this section, we establish the fundamental notation to be employed throughout the paper and properly formulate the main problem addressed in this paper.

\subsection{Preliminaries}
\label{subsec:prelim}
An undirected graph $G=(V, E)$ represents a multi-agent network, where the vertex set $V$ represents agents and the edge set $E \subseteq V \times V$ denotes interactions between them. An edge between vertices $u$ and $v$ is denoted by the unordered pair $(u,v)$. The \emph{neighborhood} of vertex $u$ is defined as $\mathcal{N}_G(u) = \{v \in V : (u,v) \in E\}$, and the \emph{degree} of $u$ is $\deg(u) = |\mathcal{N}_G(u)|$. The \emph{average degree} $\bar{d}$ is given by \(
\bar{d} = \frac{1}{|V|} \sum_{v \in V} \deg(v),\)
where $\deg(v)$ is the degree of vertex $v$.

A \emph{path} $P$ in $G$ is a sequence of distinct vertices $(v_1, v_2, \ldots, v_k)$ such that for each $i$ from $1$ to $k-1$, there exists an edge between $v_i$ and $v_{i+1}$. The \emph{distance} between vertices $u$ and $v$, denoted $d_G(u,v)$, is the number of edges in the shortest path between $u$ and $v$. For simplicity, the subscript is dropped when the context is clear. A graph $\hat{G} = (\hat{V}, \hat{E})$ is a \emph{subgraph} of $G = (V, E)$, denoted $\hat{G} \subseteq G$, if $\hat{V} \subseteq V$ and $\hat{E} \subseteq E$.

A \emph{connected component} $C$ of $G$ is a maximal subset of vertices $V' \subseteq V$ such that for every pair of vertices $u, v$ in $V'$, there exists a path between $u$ and $v$. A \emph{tree} is an undirected graph that is connected and acyclic, or equivalently, a graph with $n$ vertices and $n-1$ edges.

\subsection{Problem Formulation}
\label{subsec:problem}

In the context of graph-based machine learning, the graph classification problem is a fundamental task. The objective is to assign a discrete label to an entire graph, indicating the class to which the graph belongs. This task finds applications in various domains, such as cheminformatics, where graphs represent molecules, and social network analysis, where graphs represent interactions between individuals~\cite{hamilton2017inductive}.

Formally, given a collection of graphs $\{G_1, G_2, \ldots, G_k\}$, where each graph $G_i = (V_i, E_i)$ consists of a set of vertices $V_i$ and a set of edges $E_i$, and a corresponding set of labels $\{y_1, y_2, \ldots, y_k\}$ with $y_i \in \{0, 1, \ldots, C\}$, the goal is to learn a function $\phi: \mathcal{G} \rightarrow y$, where $\mathcal{G}$ is the set of all possible graphs, $y \in \{0, 1, \ldots, C\}$, and $C \in \mathbb{Z}$ is the number of possible labels. The function $\phi$ takes an input graph $G_i$ and outputs a label $\tilde{y}_i$, representing the predicted class of the graph. A machine learning approach to this problem involves training a model to generate this discrete labeling: a model $\phi(G_i, \bm{\theta})$ that takes an input graph $G_i$ and outputs a probability score $\phi: \{G_i\} \rightarrow [0,1]^C$ indicating the likelihood of each graph being classified as one of the classes, where $\bm{\theta}$ are the learnable weights. The learned function $\phi(G, \bm{\theta})$ should minimize the classification error $\mathcal{L}(y, \tilde{y})$ on a given set of graphs, where the error is essentially the difference between the predicted label $\tilde{y}$ and the true given label $y$.

\vspace{2 mm}
\noindent\fbox{\begin{varwidth}{\dimexpr\linewidth-1\fboxsep-2\fboxrule\relax}
    \begin{prob}
    \label{prob:Graph_classification}
    \textbf{(Graph Classification):}
    Given a graph \( G = (V, E) \), the goal is to map \( G \) to a discrete label \( y \in \{0, 1, \ldots, C\}\) using a machine learning model with learnable weights \( \bm{\theta} \). The mapping function \( \phi \) can be expressed as:
    \[ \tilde{y} = \phi(G; \bm{\theta}), \]
    where \( \tilde{y} \) is the predicted label for the graph \( G \) and \( \bm{\theta} \) represents the parameters of the model.
    \end{prob}
\end{varwidth}}
\vspace{2 mm}

In many applications of graph-based machine learning, dealing with large and dense graphs can pose significant computational challenges. Graph sparsification is a crucial technique to address this issue, aiming to reduce the number of edges in a graph while preserving its essential properties~\cite{tsitsulin2023graph}. By simplifying the graph structure, sparsification can lead to more efficient algorithms, reduced memory usage, and faster processing times, without significantly compromising the performance of graph-based tasks such as classification. Formally, given a graph \( G = (V, E) \), a sparsification function \( \mathcal{A} \) produces a sparser subgraph \( \hat{G} = (V, \hat{E}) \) such that \( \hat{E} \subseteq E \) and \( |\hat{E}| \ll |E| \). The goal is to ensure that \( \hat{G} \) retains the key structural properties of \( G \) necessary for downstream machine learning tasks.

\vspace{2 mm}
\noindent\fbox{\begin{varwidth}{\dimexpr\linewidth-1\fboxsep-2\fboxrule\relax}
    \begin{prob}
    \label{prob:sparsification}
    \textbf{(Graph Sparsification):}
    Given a graph \( G = (V, E) \) and a label \( y \), the goal is to find a sparsification function $\mathcal{A}: \mathcal{G} \rightarrow \hat{\mathcal{G}}$ such that \( G \mapsto \hat{G} = (V, \hat{E}) \), $\hat{E} \subseteq E$, and \( \phi(\hat{G}, \bm{\theta}) = \tilde{y} \) where the predicted label \( \tilde{y} \) should be the same as the given label \( y \).
    \end{prob}
\end{varwidth}}
\vspace{2 mm}

In light of graph classification, the problem can be framed to incorporate graph sparsification. To leverage sparsification, we first apply a sparsification function \( \mathcal{A} \) to each graph \( G_i \), obtaining a sparser graph \( \hat{G}_i = \mathcal{A}(G_i) \quad \forall i \in \{1, 2, \ldots, k\}\). Then, we learn the classification function \( \phi \) on the set of sparser graphs \( \{\hat{G}_1, \hat{G}_2, \ldots, \hat{G}_k\} \). The objective is to minimize the classification error $\mathcal{L}(y, \tilde{y})$ on the sparsified graphs, ensuring \(\phi(\hat{G}, \bm{\theta}) = \tilde{y} \approx y\) and thus maintaining high classification performance on the original graph set. This idea is presented in Figure \ref{fig:main_problem} where the gray box represents the main focus of this work.

\begin{figure}
\vspace{-5mm}
    \centering
    \includegraphics[scale = 0.45]{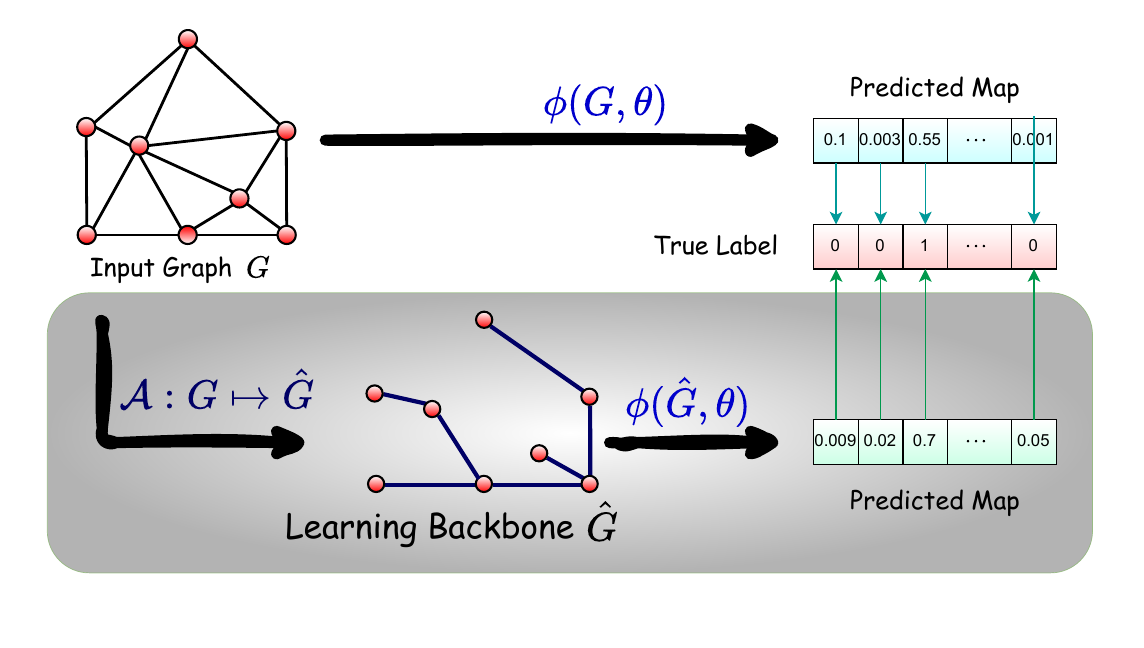}
    \caption{Main Idea: Sparsify the graph while maintaining the critical learning backbone for downstream machine learning tasks such as graph classification. The predicted label $\phi(\hat{G}, \bm{\theta}) = \tilde{y}$ should be close to the true label $y$.}
    \label{fig:main_problem}
\vspace{-5mm}
\end{figure}

In Section \ref{subsec:zfs_backbone}, we propose a novel approach to sparsify a graph for computing the learning backbone. The sparse graph we propose is a tree, as it is the minimum connected graph structure. Finding a tree graph for a given graph is not a trivial task, as it can be computationally expensive to identify a suitable tree that preserves the essential properties of the original graph required for downstream learning tasks. The exact number of spanning trees of a given graph can be computed by the Matrix Tree Theorem~\cite{godsil2001algebraic}. The number of spanning trees of a given graph \( G = (V, E) \) is the normalized product of the non-zero eigenvalues of the Laplacian matrix, and it can be as large as 
\(
\left( \frac{2m - \Delta - \delta - 1}{n-3} \right)^{n-3},
\)
where \( n \) is the number of nodes and \( n > 3 \), \( m \) is the number of edges, \( \Delta \) is the maximum degree, and \( \delta \) is the minimum degree~\cite{li2010number}. Note that the number of spanning trees in a given graph can be exponential with respect to the number of vertices.

In the next section, we present our proposed approach for finding a learning tree backbone.

\section{Learning Backbone}
\label{sec:NCB}

Control theory is a branch of engineering and mathematics focused on the behavior of dynamical systems with inputs. Its goal is to develop a control strategy that governs the system's output by manipulating the inputs. A core concept in control theory is feedback, where the system's output is measured and used to adjust the inputs to maintain desired performance. Control theory has applications in robotics, aerospace, and economics. The structural properties of the underlying graph, representing the network of interconnected components, significantly influence a system's dynamic behavior. For instance, the presence or absence of specific edges can affect the stability, controllability, and observability of the system~\cite{abbas2023zero}.

The dynamic behavior of a control system is closely related to graph-based machine learning. In machine learning, how information propagates through a graph is crucial for tasks such as node classification, link prediction, and graph classification~\cite{said2023network, said2024improving}. The graph structure dictates how signals spread across the network, influencing the performance of graph neural networks (GNNs) and other models. Sparse representations, like trees, play a vital role by preserving essential control properties, such as connectivity and controllability~\cite{ahmad2023controllability}, while reducing computational complexity. This approach aligns with the Graph Lottery Ticket Hypothesis, which suggests that a sparse substructure within a complex graph can achieve comparable performance to the original, optimizing both control and learning objectives.

In this section, we explore the concept of the network controllability backbone, which aims to identify a sparse subset of edges that preserves the network's controllability under structural perturbations. We begin by establishing the fundamental framework for understanding controllability in networked systems and then introduce the concept of a learning backbone.

\subsection{Controllability Framework}
Consider a network of \( n \) agents, denoted by \( V = \{v_1, v_2, \cdots, v_n\} \). Among these agents, \( m \) are designated as input or leader vertices, represented as \( V_\ell = \{\ell_1, \ell_2, \cdots, \ell_m\} \subseteq V \), while the remaining vertices act as followers. The network's dynamics are modeled by the following linear time-invariant system:
\begin{equation}
\label{eq:system}
\dot{x}(t) = Mx(t) + Hu(t),
\end{equation}
where \( x(t) \in \mathbb{R}^n \) is the state vector, and \( u(t) \in \mathbb{R}^m \) represents the external input injected through the \( m \) leaders. The matrix \( M \in \mathcal{M}(G) \) is the system matrix associated with the graph \( G \), and \( H \in \mathbb{R}^{n \times m} \) is the input matrix determined by the leader vertices. The family of matrices \( \mathcal{M}(G) \) is defined as follows:
\begin{equation}
\label{eq:mtx_family}
\begin{split}
\mathcal{M}(G) = \{M \in \mathbb{R}^{n\times n} \;& : \;M = M^\top, \text{ and for } i \neq j, \\
& M_{ij} \neq 0 \Leftrightarrow (i,j) \in E(G)\}.
\end{split}
\end{equation}
This definition encompasses a broad class of system matrices associated with the graph \( G \), including the adjacency matrix, Laplacian matrix, and the signless Laplacian matrix.

A system \eqref{eq:system} is \emph{controllable} if an input \( u(t) \) can drive the system from any initial state \( x(t_0) \) to any desired state \( x(t_f) \) in finite time. We say that \( (M, H) \) is a \emph{controllable pair} if and only if the controllability matrix \( \mathcal{C}(M, H) \in \mathbb{R}^{n \times nm} \) is full rank, i.e., \( \texttt{rank}(\mathcal{C}(M, H)) = n \). The controllability matrix is given by:
\begin{equation}
\label{eq:Gamma}
\mathcal{C}(M, H) = \left[ H \quad MH \quad M^2H \quad \cdots \quad M^{n-1}H \right].
\end{equation}

\begin{definition} (Strong Structural Controllability (SSC))
A graph \( G = (V, E) \) with a specified set of leaders \( V_\ell \subseteq V \) (and the corresponding \( H \) matrix) is \emph{strongly structurally controllable} if and only if \( (M, H) \) is a controllable pair for all \( M \in \mathcal{M}(G) \).
\end{definition}

If the network \( G \) is strongly structurally controllable for a given set of leaders, then the rank of the controllability matrix does not depend on the edge weights (as long as they satisfy the conditions given by \(\mathcal{M}(G)\)). For the remainder of this paper, we will refer to strong structural controllability simply as \emph{controllability}. The dimension of the strongly structurally controllable subspace, denoted by \( \gamma(G, V_\ell) \), is the smallest possible rank of the controllability matrix under feasible edge weights.

\subsubsection*{Network Controllability Backbone}

The main idea of a controllability backbone is to identify a minimal subset of edges within a network that ensures the preservation of its controllability in any subgraph. We define the \emph{controllability backbone} as this sparse subgraph, denoted by $B$, such that any subgraph $\hat{G}$ containing $B$ maintains at least the same level of controllability as the original network $G$.

\begin{definition} (Controllability Backbone)
For a given graph \( G = (V, E) \) and a set of leaders \( V_\ell \subseteq V \), the controllability backbone \( B = (V, E_{B}) \) is a subgraph of \( G \) such that any subgraph \( \hat{G} = (V, \hat{E}) \) containing \( E_{B} \), i.e., \( E_{B} \subseteq \hat{E} \subseteq E \), satisfies:
\begin{equation}
\gamma(\hat{G}, V_\ell) \geq \gamma(G, V_\ell).
\end{equation}
\end{definition}

In essence, the controllability backbone ensures that the controllability of any subgraph encompassing it does not deteriorate compared to the original network. 

\subsection{Zero Forcing for Controllability Backbone}
\label{subsec:zfs_backbone}

Zero forcing is a rule-based coloring technique for vertices in a graph, providing a lower bound on the dimension of Strong Structural Controllability (SSC). By leveraging zero forcing, our aim is to identify a subset of edges constituting the controllability backbone, termed as the ZFS-based backbone.

\begin{definition}[\emph{Zero Forcing (ZF) Process}]
Let \( G = (V, E) \) be a graph where each vertex \( v \in V \) is initially colored either \texttt{black} or \texttt{white}. The ZF process iteratively changes the color of \texttt{white} vertices to \texttt{black} according to the following rule until no further color changes are possible: \emph{Color change rule: If a \texttt{black} vertex \( v \in V \) has exactly one \texttt{white} neighbor \( u \), change the color of \( u \) to \texttt{black}.}
\end{definition}

We define a \emph{forced} relationship between vertices \( v \) and \( u \) if a \texttt{black} vertex \( v \) changes the color of a \texttt{white} vertex \( u \) to \texttt{black} during the ZF process.

\begin{definition}[\emph{Derived Set}]
Let \( G = (V, E) \) be a graph with \( V_\ell \subseteq V \) representing the initial set of \texttt{black} vertices. The derived set~\cite{work2008zero}, denoted by \( dset(G, V_\ell) \), is the set of \texttt{black} vertices obtained after the ZF process, and \( |dset(G, V_\ell)| = \zeta(G, V_\ell) \). When the context is clear, we omit the parameter \( V_\ell \).
\end{definition}

The set of initial \texttt{black} vertices \( V_\ell \) is also known as the \emph{input or leader set}. For a given \( V_\ell \), \( dset(G, V_\ell) \) is unique~\cite{work2008zero}. Now, we define the zero forcing set.

\begin{definition}[\emph{Zero Forcing Set (ZFS)}]
For a graph \( G = (V, E) \), \( V_\ell \subseteq V \) is a ZFS if and only if \( dset(G, V_\ell) = V \). We denote a ZFS of \( G \) by \( Z(G) \). 
\end{definition}

Figure~\ref{fig:ZFS} illustrates zero forcing through a set of input vertices and the corresponding derived set. Initially, \( V_\ell = \{v_1, v_2, v_5, v_6\} \) are colored black. In the next step, \( v_2 \) can force \( v_3 \) as it is its only white neighbor and so on.

\begin{figure}[ht]
	\centering
	\includegraphics[scale=0.60]{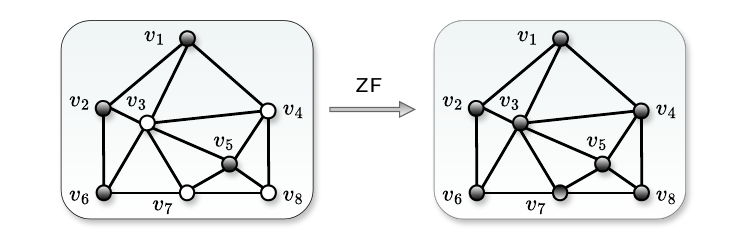}
    \caption{\( V_\ell = \{v_1, v_2, v_5, v_6\} \) is the input set. After the ZF process, \( dset(G, V_\ell) = V \), as indicated by the \texttt{black} vertices. Hence, \( V_\ell \) is a ZFS.}
    \label{fig:ZFS}
\end{figure}

The zero forcing phenomenon is significant in characterizing the network's SSC~\cite{
yaziciouglu2022strong}. In particular, the size of the derived set for a given set of input vertices provides a lower bound on the dimension of SSC, i.e., for a network \( G = (V, E) \) with the leader set \( V_\ell \subseteq V \), we have \( \zeta(G, V_\ell) \leq \gamma(G, V_\ell) \)~\cite{ahmad2023graph,
yaziciouglu2022strong
}. By computing the ZFS, we obtain a lower bound on the dimension of SSC, facilitating the identification of the controllability backbone.

\subsubsection*{ZFS-based Backbone}

Our goal is to discover a backbone that maintains the zero forcing bound $\zeta(G, V_\ell)$ for a given leader set $V_\ell$. The objective is to identify a subset of edges $E_{B_Z}$ in the graph $G = (V, E)$ with $V_\ell$ such that the ZFS-based controllability bound is preserved in any subgraph $\hat{G} = (V, \hat{E})$ containing those edges ($E_{B_Z} \subseteq \hat{E}$). Formally, we define the ZFS-based backbone as follows:

\begin{definition} \label{def:zfs_bckbn} (\emph{ZFS-based Backbone}) Given a graph $G = (V, E)$ and a leader set $V_\ell$, the ZFS-based backbone, denoted as $B_z = (V, E_{B_z})$, is a subgraph where any subgraph $\hat{G} = (V, \hat{E})$ containing $E_{B_Z}$ satisfies \(\zeta(\hat{G}, V_\ell) \ge \zeta(G, V_\ell).\)
\end{definition}

Thus, in any subgraph of $G$ containing the ZFS-based backbone, the dimension of SSC is at least $\zeta(G, V_\ell)$, i.e., $\gamma(\hat{G}, V_\ell) \ge \zeta(G, V_\ell)$.

\subsection{Controllability Backbone as Learning Backbone}

In graph-based learning, optimizing the underlying graph structure for both control and information propagation is crucial. The Zero Forcing Set (ZFS) method is a powerful tool for maintaining network controllability. By leveraging the principles of strong structural controllability (SSC), we can identify a minimal subset of edges, termed the ZFS-based controllability backbone, which preserves the essential control properties of the original graph. This backbone, effectively forming a sparse substructure, ensures robust dynamic behavior while significantly reducing computational complexity. Using the ZFS-based controllability backbone as the learning backbone aims to enhance the efficiency of graph classification tasks while maintaining the critical control properties of the original network.

It has been shown that the ZFS-based backbone $B_z$ is a set of paths originating from vertices $v\in V_\ell$, always having $n - |V_\ell|$ edges, and consequently, $|V_\ell|$ connected components \cite{ahmad2023controllability}. If the original graph is connected, edges can be added to form a subgraph $G^\prime = (V, E^\prime)$ such that $E_{B_Z} \subseteq E^\prime \subseteq E$, and $G^\prime$ is a connected tree. This tree, known as the learning backbone, substitutes the original tree and can be used for downstream machine-learning tasks. This approach is presented in Algorithm \ref{alg:backbone_zfs}.

\begin{algorithm}[ht]
\caption{Computing Learning Backbone}
\label{alg:backbone_zfs}
    \begin{algorithmic}[1]
    \renewcommand{\algorithmicrequire}{\textbf{Input:}}
    \renewcommand{\algorithmicensure}{\textbf{Output:}}
    \Require Graph $G = (V, E)$
    \Ensure Learning backbone $\hat{G} = (V, \hat{E})$
        \State Compute a zero-forcing set $V_\ell$ \cite{ahmad2023graph}
        \State Initialize a graph $B_z$ with paths originating from all vertices $v \in V_\ell$ by running the zero-forcing process
        \State Add any $|V_\ell| -1$ edges from $G$ to $B_z$ to form $\hat{G}$ such that $\hat{G}$ becomes a connected tree
    \end{algorithmic}
\end{algorithm}

\begin{theorem}
\label{thm:zfs_backbone}
Given a graph $G=(V,E)$, Algorithm~\ref{alg:backbone_zfs} returns a learning backbone, a \emph{connected tree}, that is strong structurally controllable for the computed leader set $V_\ell$.
\end{theorem}

\begin{proof}
    For any given graph $G = (V, E)$ and leader set $V_\ell$, any subgraph $\hat{G} = (V,\hat{E})$, where  $E_{B_Z} \subseteq \hat{E} \subseteq E$, satisfies the relation $$\zeta(\hat{G}, V_\ell) \ge \zeta(G, V_\ell)$$ by definition. In step 1 of Algorithm \ref{alg:backbone_zfs}, we compute a zero-forcing set that makes the graph fully controllable. Hence, $\gamma(G, V_\ell) = |V|$. The ZFS-based backbone $B_z$ contains an unconnected set of paths where each path originates from a leader vertex. $B_z$ can be computed from Algorithm 1 of \cite{ahmad2023controllability}. By definition of $B_z$, we can add any number of edges from the original graph randomly, and the graph will remain fully controllable. The learning backbone $\hat{G}$ contains only the edges that are in the original graph besides containing the controllability backbone $B_z$. Hence, we can compute a connected tree with $n-1$ edges from Algorithm \ref{alg:backbone_zfs} where the tree would be strong structurally controllable for the computed zero-forcing set $V_\ell$.
\end{proof}

\begin{figure*}
    \vspace{-0.8cm}
    \centering
    \includegraphics[width = \textwidth]{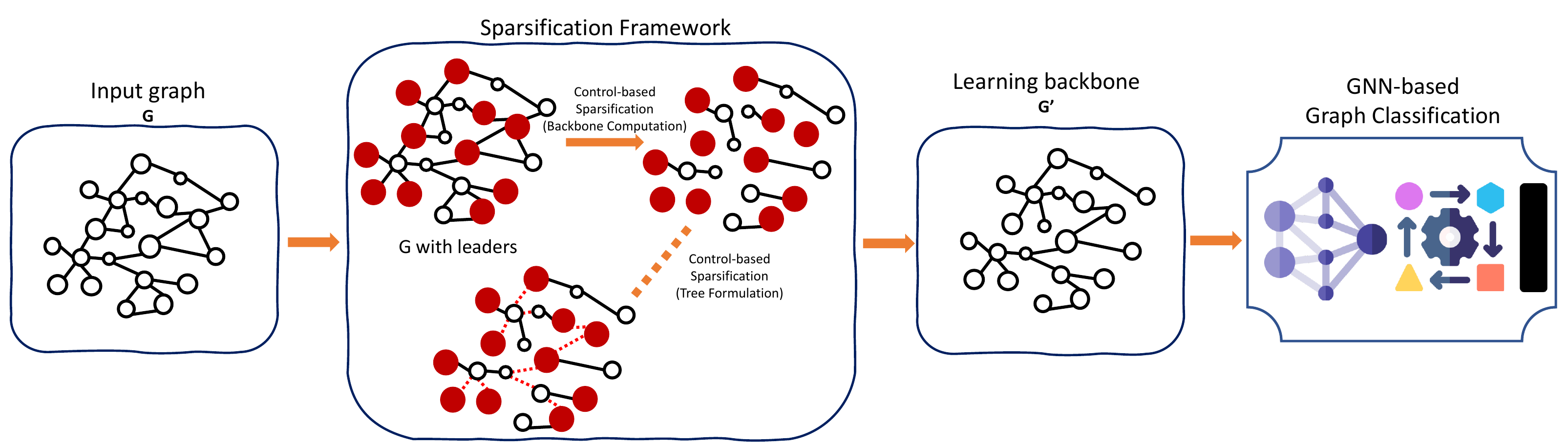}
    \caption{Illustration of the proposed framework: The process begins by identifying a leader vertex subset within the graph. Using network control theory, a graph sparsification framework is then applied to derive a tree-like structure, called the `learning backbone', from the original graph.}
    \label{fig:block-dia}
    \vspace{-0.3cm}
\end{figure*}

\subsection{Generalized Learning Backbone}
\label{subsec:gen_backbone}

In networked systems, preserving various control properties is crucial for robust performance and effective information propagation. While the Zero Forcing Set (ZFS) method ensures network controllability, other control properties, such as controllability matrices, graph distances, and structural patterns, also significantly impact dynamical system behavior.

Graph distances, representing the shortest paths between vertices, are particularly important for understanding how control signals propagate through a network. The distribution of these distances influences system stability and responsiveness~\cite{yaziciouglu2022strong}. In closely connected networks, control inputs more efficiently affect the entire system. The distances between all vertices and leaders also determine the lower bound on controllability rank, $\gamma(G, V_\ell)$~\cite{yaziciouglu2016graph}. For a network \(G = (V,E)\) with leaders \(V_\ell = \{\ell_1, \ell_2, \cdots , \ell_m\}\), the \emph{distance-to-leader (DL) vector} for each \(v_i \in V\) is defined as
\[
D_i = \left[\begin{array}{lllll}
    d(\ell_1, v_i) & d(\ell_2, v_i) & \cdots & d(\ell_{m}, v_i)
\end{array}\right]^T \in \mathbb{Z}^m,
\]
where \([D_i]_j = d(\ell_j, v_i)\) is the distance between leader \(\ell_j\) and vertex \(v_i\). The maximum sequence of these DL vectors that meets certain constraints defines the lower bound on \(\gamma(G, V_\ell)\).

Ahmad et al. introduced the distance-based controllability backbone $B_d = (V, E_{B_d})$, which emphasizes preserving key distances between vertex pairs, unlike the ZFS-based backbone $B_z$, which focuses on tree-like structures~\cite{ahmad2023controllability}. While $B_z$ ensures controllability via paths, $B_d$ maintains critical distances while preserving sparsity with $O(n)$ edges, where $n$ is the number of vertices.

Incorporating graph distances into backbone construction enhances control properties and ensures robust dynamic behavior, supporting downstream tasks like graph classification by retaining the network's structural integrity, as demonstrated in Section \ref{sec:results}.

In summary, while the ZFS-based method is valuable for ensuring controllability, considering additional control properties like graph distances offers a more comprehensive approach. The distance-based backbone balances sparsity with the preservation of critical features, providing robust control across various applications. We incorporate the distance-based backbone in our experiments, detailed in Section \ref{sec:results}.

\section{Experimental Results}
\label{sec:results}

In this section, we offer a comprehensive evaluation of the proposed framework within the context of graph classification tasks, employing real-world social networks and molecular datasets. We introduce the datasets and then provide a detailed description of the experimental setup. Following the setup, we discuss the results, elucidating the efficacy of our framework.

\begin{table*}[!t]
\centering
\caption{Dataset stats}
\begin{tabular}{|c|c|cc|cc|cccc|}
\hline
\multirow{3}{*}{Dataset} & \multirow{3}{*}{\# of Graphs} & \multicolumn{2}{c|}{\multirow{2}{*}{\# of Nodes}} & \multicolumn{2}{c|}{\multirow{2}{*}{Average Degree $\bar{d}$}} & \multicolumn{4}{c|}{Density}                                                                 \\ \cline{7-10} 
                         &                               & \multicolumn{2}{c|}{}                             & \multicolumn{2}{c|}{}                                          & \multicolumn{2}{c|}{Original}                           & \multicolumn{2}{c|}{Backbone}      \\ \cline{3-10} 
                         &                               & \multicolumn{1}{c|}{min}           & max          & \multicolumn{1}{c|}{Original}            & Backbone            & \multicolumn{1}{c|}{min}   & \multicolumn{1}{c|}{max}   & \multicolumn{1}{c|}{min}   & max   \\ \hline
MUTAG                    & 188                           & \multicolumn{1}{c|}{10}            & 28           & \multicolumn{1}{c|}{2.189}                     & 1.88                      & \multicolumn{1}{c|}{0.082} & \multicolumn{1}{c|}{0.222} & \multicolumn{1}{c|}{0.071} & 0.20  \\ 
PTC                      & 344                           & \multicolumn{1}{c|}{2}             & 64           & \multicolumn{1}{c|}{1.981}                     & 1.862                     & \multicolumn{1}{c|}{0.034} & \multicolumn{1}{c|}{1.0}   & \multicolumn{1}{c|}{0.031} & 1.0   \\ 
PROTEINS                 & 1113                          & \multicolumn{1}{c|}{4}             & 620          & \multicolumn{1}{c|}{3.735}                     & 1.893                     & \multicolumn{1}{c|}{0.005} & \multicolumn{1}{c|}{1.0}   & \multicolumn{1}{c|}{0.003} & 0.5   \\ 
NCI1                     & 4110                          & \multicolumn{1}{c|}{3}             & 111          & \multicolumn{1}{c|}{2.155}                     & 1.908                     & \multicolumn{1}{c|}{0.019} & \multicolumn{1}{c|}{0.667} & \multicolumn{1}{c|}{0.018} & 0.667 \\ 
Deezer Ego               & 9,629                         & \multicolumn{1}{c|}{11}            & 363          & \multicolumn{1}{c|}{4.292}                     & 1.887                     & \multicolumn{1}{c|}{0.015} & \multicolumn{1}{c|}{0.909} & \multicolumn{1}{c|}{0.006} & 0.182 \\ 

GitHub Stargazers        & 12,725                        & \multicolumn{1}{c|}{10}            & 957          & \multicolumn{1}{c|}{3.111}                     & 1.939                     & \multicolumn{1}{c|}{0.003} & \multicolumn{1}{c|}{0.561} & \multicolumn{1}{c|}{0.004} & 0.200 \\ 
Twitch Ego               & 127,094                       & \multicolumn{1}{c|}{12}            & 52           & \multicolumn{1}{c|}{5.397}                     & 1.922                     & \multicolumn{1}{c|}{0.038} & \multicolumn{1}{c|}{0.967} & \multicolumn{1}{c|}{0.038} & 0.143 \\ 
Reddit Threads           & 203,088                       & \multicolumn{1}{c|}{11}            & 97           & \multicolumn{1}{c|}{2.039}                     & 1.889                     & \multicolumn{1}{c|}{0.021} & \multicolumn{1}{c|}{0.328} & \multicolumn{1}{c|}{0.021} & 0.182 \\ 
\hline
\end{tabular}
\vspace{-0.5cm}
\label{tab:dataset-stats}
\end{table*}

\subsection{Datasets}

We evaluate our proposed approach using eight real-world datasets relevant for binary graph classification tasks. These datasets include MUTAG, PTC, PROTEINS, NCI1, Deezer Ego Network, GitHub Stargazers, Twitch Ego Networks, and Reddit Threads~\cite{morris2020tudataset, karateclub}. Each dataset presents unique challenges in graph classification, offering a comprehensive testbed for assessing the effectiveness of our ZFS-based backbone approach.

\subsection{Experimental Setup}

We evaluate six widely recognized graph convolution methods: $k$-GNN, GraphSAGE, GCN, Transformer Convolution (UniMP), Residual Gated Graph ConvNets (ResGatedGCN), and Graph Attention Network (GAT). The proposed learning frameworks consists of three GNN layers, each with 64 hidden units. After the GNN layers, we apply Sort Aggregation, followed by two 1D convolution layers with Max Pooling. The output is then passed through a two-layer multi-layer perceptron, each layer containing 32 hidden neurons.

For evaluation, we perform 10-fold cross-validation, training each model for 100 epochs. The learning rate is set to $1 \times 10^{-4}$, and weight decay is $5 \times 10^{-4}$. All experiments are conducted on a Lambda machine with an AMD Ryzen Threadripper PRO 3975WX 32-Core CPU, 512 GB of RAM, and an NVIDIA RTX 3090 GPU with 16 GB of memory.

We present the ROC AUC (Receiver Operating Characteristic Area Under the Curve) classification results for all eight datasets, comparing the original graphs and the ZFS-based tree backbone graphs in Table \ref{tab:results_zfs}. Consistent architectures and experimental settings are used for evaluation. Overall, the performance between the backbone graphs $B_z$ and the original graphs is comparable across all datasets.

\subsection{Results Analysis}

\begin{table*}[!t]
\centering
\caption{Comparison of ROC AUC scores of the proposed method (ZFS-based backbone) against original graphs. Pairs where the backbone ROC AUC is within 5\% of the original are highlighted in blue. Additionally, backbone values higher than the original are \textbf{bolded}.}
\resizebox{\textwidth}{!}{%
\begin{tabular}{|l|cc|cc|cc|cc|cc|cc|}
\hline
\multirow{2}{*}{Datasets} & \multicolumn{2}{c|}{\textbf{k-GNN}}      & \multicolumn{2}{c|}{\textbf{SAGE}}        & \multicolumn{2}{c|}{\textbf{GCN}}   & \multicolumn{2}{c|}{\textbf{UniMP}}   & \multicolumn{2}{c|}{\textbf{ResGatedGCN}}   & \multicolumn{2}{c|}{\textbf{GAT}}     \\ \cline{2-13} 
                          & \multicolumn{1}{c|}{Original} & Backbone & \multicolumn{1}{c|}{Original} & Backbone & \multicolumn{1}{c|}{Original} & Backbone & \multicolumn{1}{c|}{Original} & Backbone & \multicolumn{1}{c|}{Original} & Backbone & \multicolumn{1}{c|}{Original} & Backbone   \\ \hline

Deezer Ego    & \multicolumn{1}{c|}{\cellcolor{blue!05}50.98} & \cellcolor{blue!05}\textbf{52.07} & \multicolumn{1}{c|}{\cellcolor{blue!05}50.82} & \cellcolor{blue!05}\textbf{52.54} & \multicolumn{1}{c|}{\cellcolor{blue!05}48.45} & \cellcolor{blue!05}\textbf{52.66} & \multicolumn{1}{c|}{\cellcolor{blue!05}50.34} & \cellcolor{blue!05}\textbf{52.65} & \multicolumn{1}{c|}{\cellcolor{blue!05}51.72} & \cellcolor{blue!05}\textbf{54.21} & \multicolumn{1}{c|}{\cellcolor{blue!05}50.54} & \cellcolor{blue!05}\textbf{50.88} \\ 

Twitch Ego    & \multicolumn{1}{c|}{\cellcolor{blue!05}72.23} & \cellcolor{blue!05}\textbf{72.36} & \multicolumn{1}{c|}{\cellcolor{blue!05}72.34} & \cellcolor{blue!05}\textbf{72.47} & \multicolumn{1}{c|}{\cellcolor{blue!05}72.44} & \cellcolor{blue!05}72.25 & \multicolumn{1}{c|}{\cellcolor{blue!05}72.35} & \cellcolor{blue!05}\textbf{72.47} & \multicolumn{1}{c|}{\cellcolor{blue!05}72.42} & \cellcolor{blue!05}\textbf{72.49} & \multicolumn{1}{c|}{\cellcolor{blue!05}72.37} & \cellcolor{blue!05}\textbf{72.47} \\ 

GitHub Stargazers & \multicolumn{1}{c|}{\cellcolor{blue!05}71.47} & \cellcolor{blue!05}68.54 & \multicolumn{1}{c|}{\cellcolor{blue!05}64.95} & \cellcolor{blue!05}61.85 & \multicolumn{1}{c|}{\cellcolor{blue!05}65.58} & \cellcolor{blue!05}64.65 & \multicolumn{1}{c|}{\cellcolor{blue!05}65.59} & \cellcolor{blue!05}65.45 & \multicolumn{1}{c|}{\cellcolor{blue!05}72.55} & \cellcolor{blue!05}68.52 & \multicolumn{1}{c|}{\cellcolor{blue!05}65.01} & \cellcolor{blue!05}62.03 \\ 

Reddit Threads    & \multicolumn{1}{c|}{\cellcolor{blue!05}83.80} & \cellcolor{blue!05}83.40 & \multicolumn{1}{c|}{\cellcolor{blue!05}82.99} & \cellcolor{blue!05}\textbf{83.45} & \multicolumn{1}{c|}{\cellcolor{blue!05}83.06} & \cellcolor{blue!05}\textbf{83.26} & \multicolumn{1}{c|}{\cellcolor{blue!05}83.87} & \cellcolor{blue!05}83.54 & \multicolumn{1}{c|}{\cellcolor{blue!05}83.87} & \cellcolor{blue!05}83.55 & \multicolumn{1}{c|}{\cellcolor{blue!05}83.84} & \cellcolor{blue!05}83.05 \\ 

MUTAG            & \multicolumn{1}{c|}{\cellcolor{blue!05}93.20} & \cellcolor{blue!05}90.13 & \multicolumn{1}{c|}{\cellcolor{blue!05}86.02} & \cellcolor{blue!05}\textbf{92.95} & \multicolumn{1}{c|}{\cellcolor{blue!05}88.07} & {\cellcolor{blue!05}\textbf{92.31}} & \multicolumn{1}{c|}{91.92} & 82.94 & \multicolumn{1}{c|}{\cellcolor{blue!05}92.95} & \cellcolor{blue!05}92.95 & \multicolumn{1}{c|}{\cellcolor{blue!05}90.25} & \cellcolor{blue!05}\textbf{95.13} \\ 

PTC             & \multicolumn{1}{c|}{\cellcolor{blue!05}49.10} & \cellcolor{blue!05}\textbf{56.80} & \multicolumn{1}{c|}{4\cellcolor{blue!05}7.79} & \cellcolor{blue!05}\textbf{56.07} & \multicolumn{1}{c|}{50.93} & 46.87 & \multicolumn{1}{c|}{\cellcolor{blue!05}48.36} & {\cellcolor{blue!05}\textbf{57.40}} & \multicolumn{1}{c|}{57.53} & 48.70 & \multicolumn{1}{c|}{\cellcolor{blue!05}53.33} & {\cellcolor{blue!05}\textbf{57.07}} \\ 

PROTEINS        & \multicolumn{1}{c|}{\cellcolor{blue!05}78.65} & \cellcolor{blue!05}75.62 & \multicolumn{1}{c|}{\cellcolor{blue!05}77.59} & \cellcolor{blue!05}73.78 & \multicolumn{1}{c|}{78.37} & 72.36 & \multicolumn{1}{c|}{\cellcolor{blue!05}77.86} & \cellcolor{blue!05}76.25 & \multicolumn{1}{c|}{\cellcolor{blue!05}77.02} & \cellcolor{blue!05}75.34 & \multicolumn{1}{c|}{77.62} & 72.95 \\ 

NCI1            & \multicolumn{1}{c|}{77.78} & 69.60 & \multicolumn{1}{c|}{\cellcolor{blue!05}69.23} & \cellcolor{blue!05}67.92 & \multicolumn{1}{c|}{72.34} & 61.66 & \multicolumn{1}{c|}{\cellcolor{blue!05}70.63} & \cellcolor{blue!05}67.63 & \multicolumn{1}{c|}{72.50} & 66.41 & \multicolumn{1}{c|}{71.72} & 66.89 \\ 
\hline

\end{tabular}
}

\label{tab:results_zfs}
\end{table*}

For certain datasets, such as Deezer Ego and PTC, we observe from Table \ref{tab:results_zfs} that the ZFS-based backbone—derived from zero forcing-based controllability—serves as a more effective representation for graph learning tasks. This is exemplified by the notable performance improvement seen with the PTC dataset, where the application of a UniMP baseline GNN model on the backbone graph resulted in a maximum performance enhancement of 9.04\%. In fact, \emph{in 20 out of 48 combinations, the backbone representation resulted in a better ROC AUC compared to the original graphs}. Moreover, \emph{in 38 out of 48 combinations, the backbones exhibited less than 5\% deterioration in ROC AUC}, further underscoring the potential of our proposed backbone to not only simplify the graph structure but also to potentially uncover more salient features pertinent to the learning task.

Conversely, it is crucial to acknowledge instances where the proposed backbone representation led to a decrease in performance. The most significant reduction was observed with the NCI1 dataset, where the application of the GCN baseline model on the backbone graph saw a decline in ROC AUC by 10.68\%. This suggests that while the proposed backbone can generally maintain or improve performance, there may be specific scenarios or datasets where the full topology of the original graph is necessary to capture the nuances required for better classification.

In Section \ref{sec:NCB}, we introduced two methodologies for deriving control backbones from networks: the ZFS approach and the distance-based approach. Both methods are designed to ensure network controllability. Similar to $B_z$, the distance-based backbone, denoted as $B_d$ \cite{ahmad2023controllability}, is crafted to maintain the lower controllability bound, preserving the network's control characteristics. We evaluated the effectiveness of these backbones by comparing their performance to the original graphs, and included random spanning trees, constructed using Kruskal's algorithm, as a baseline for learning and controllability. The empirical results, shown in Figure \ref{fig:results}, reveal a clear trend: in most cases, the control backbones outperform the original graphs across various datasets and models. \emph{In 67\% of cases, the control backbones improve ROC AUC compared to the original graphs, with less than 5\% deterioration in the remaining cases}. Additionally, \emph{the control backbones outperform random spanning trees in approximately 80\% of cases, with less than 2\% deterioration in the rest}. These results demonstrate that the ZFS- and distance-based control backbones provide an effective solution for simplifying network structures while retaining essential control and learning properties.

\begin{figure*}[!t]
\centering
    \begin{subfigure}{0.235\textwidth}
        \centering
        \includegraphics[width=\linewidth]{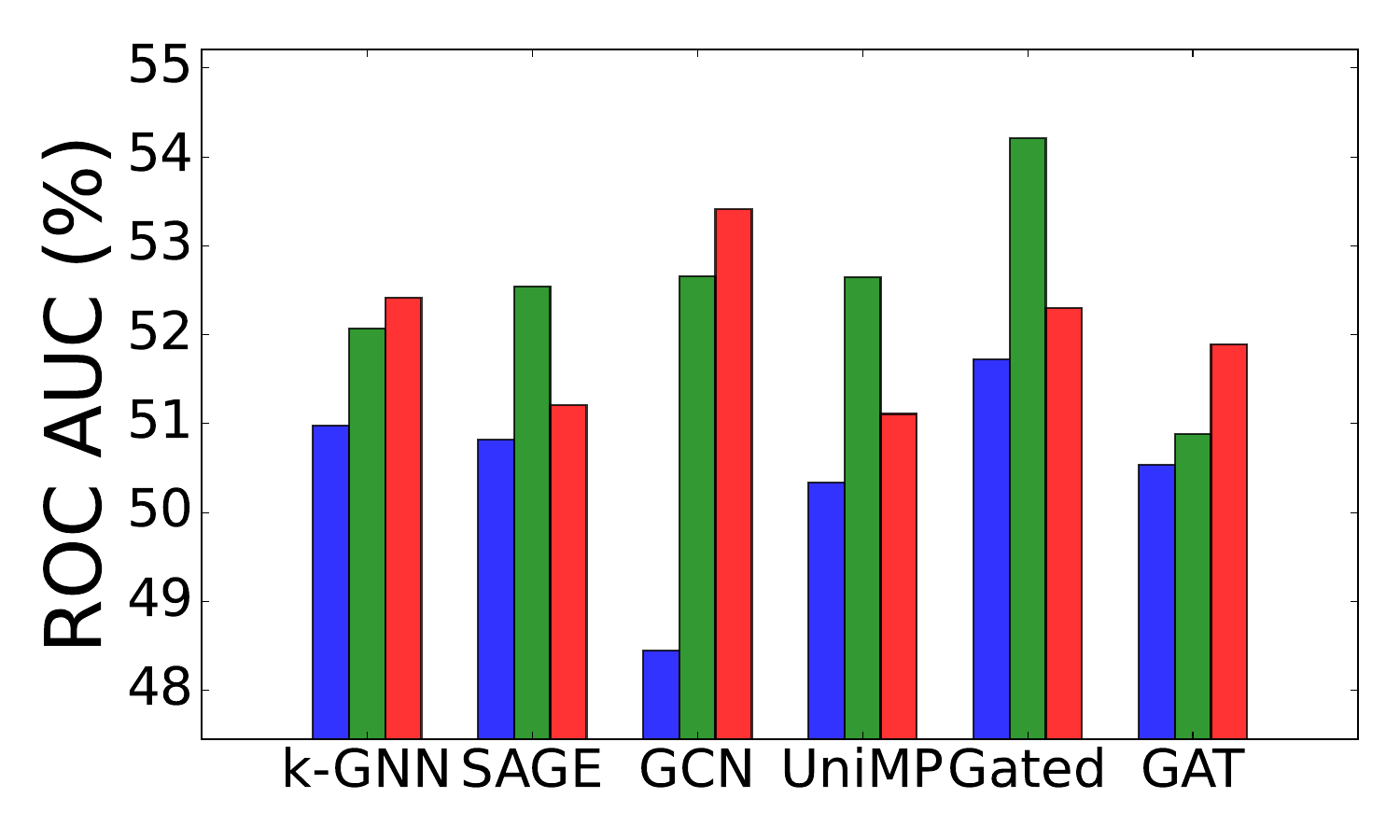}
        \caption{Deezer Ego}
    \end{subfigure}%
    \hspace{0.01\textwidth}
    \begin{subfigure}{0.235\textwidth}
        \centering
        \includegraphics[width=\linewidth]{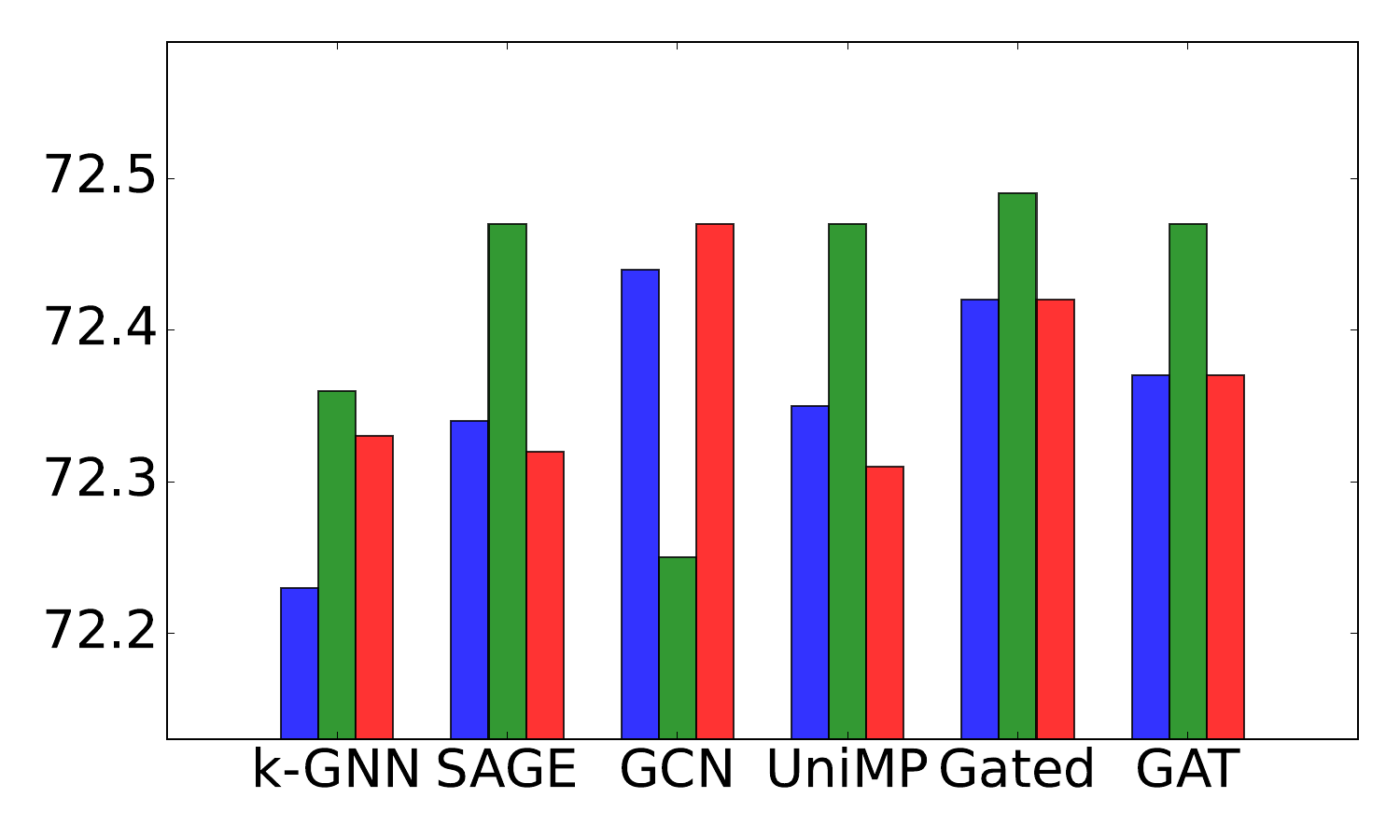}
        \caption{Twitch Ego}
    \end{subfigure}%
    \hspace{0.01\textwidth}
    \begin{subfigure}{0.235\textwidth}
        \centering
        \includegraphics[width=\linewidth]{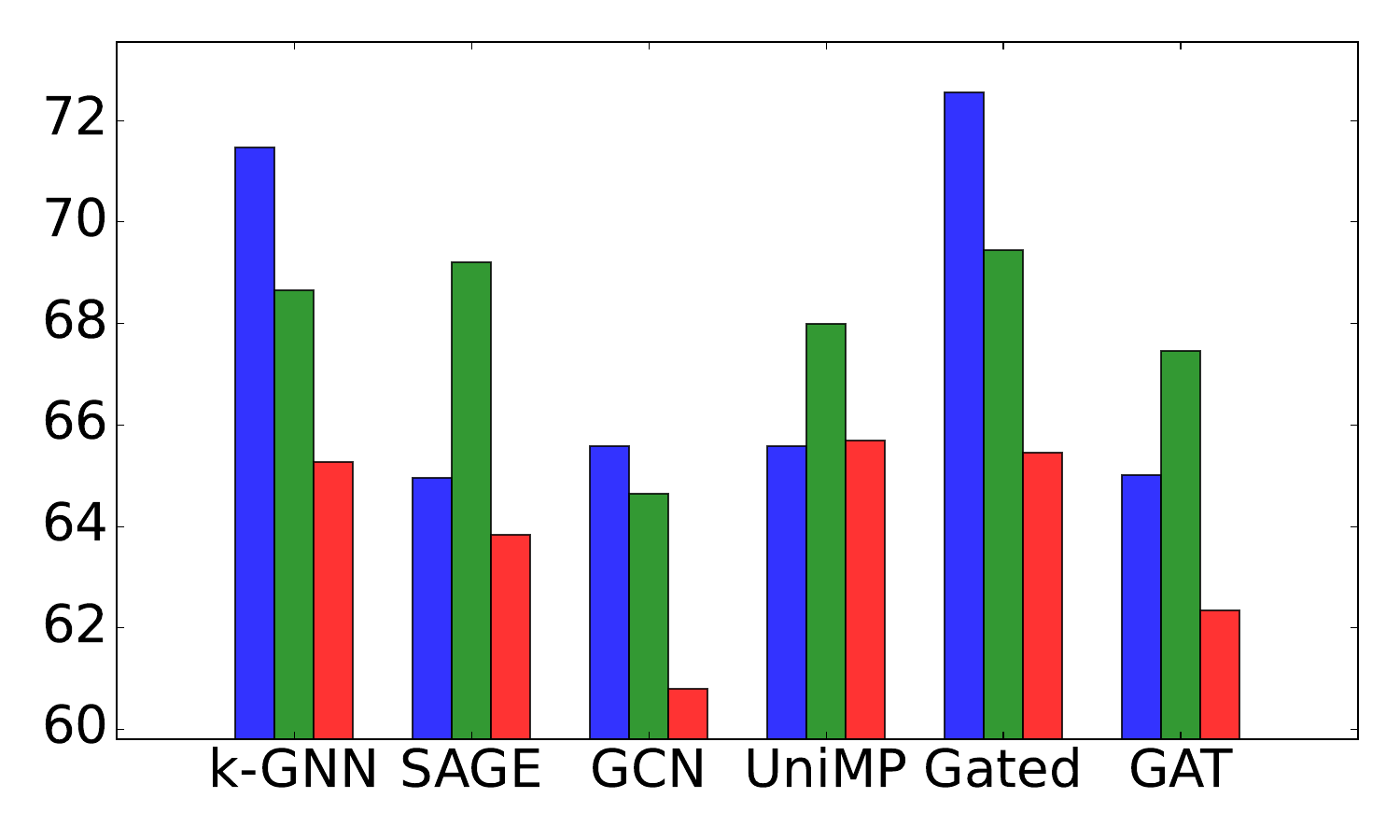}
        \caption{Github Stargazers}
    \end{subfigure}%
    \hspace{0.01\textwidth}
    \begin{subfigure}{0.235\textwidth}
        \centering
        \includegraphics[width=\linewidth]{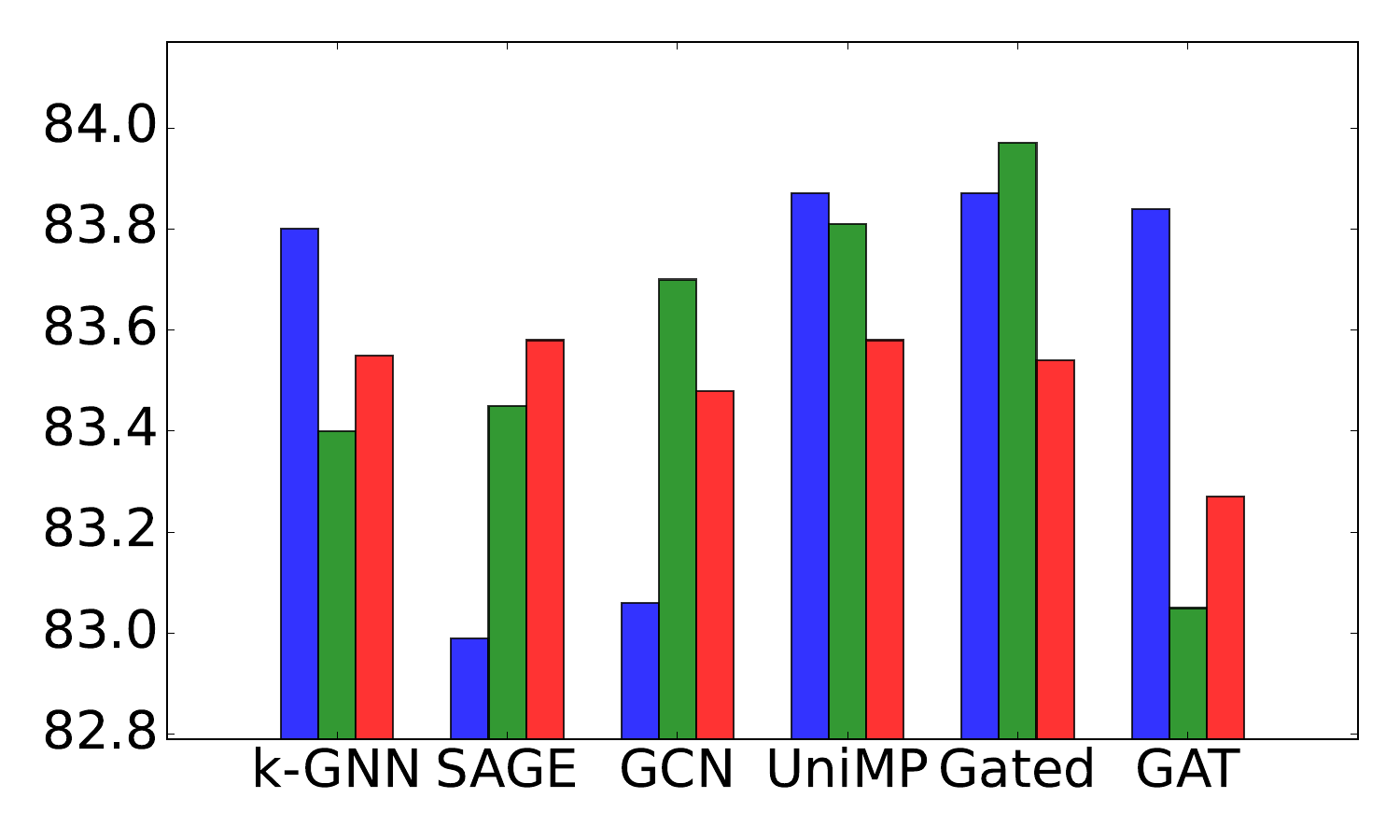}
        \caption{Reddit Threads}
    \end{subfigure}
    
    \vspace{0.5cm}
    
    \begin{subfigure}{0.235\textwidth}
        \centering
        \includegraphics[width=\linewidth]{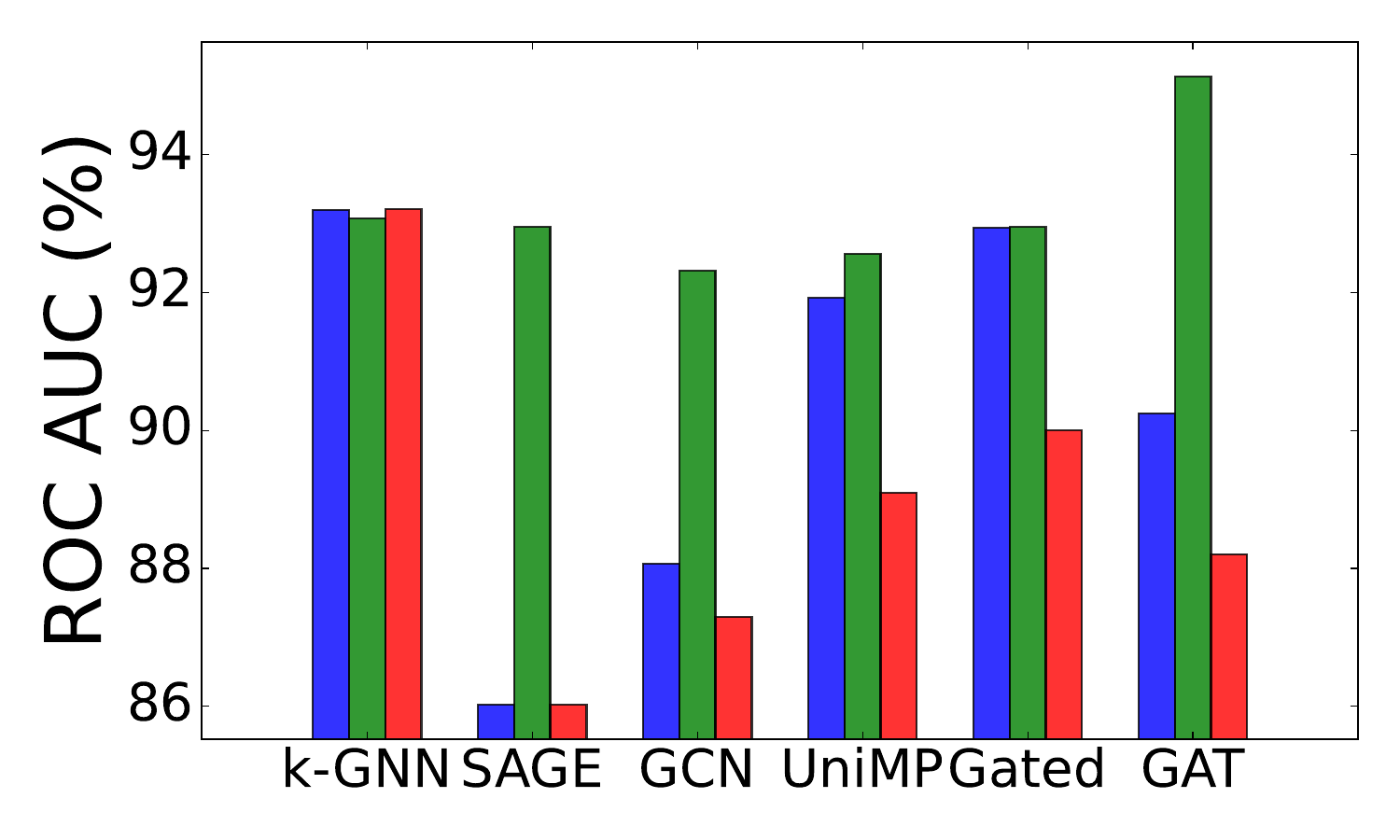}
        \caption{MUTAG}
    \end{subfigure}%
    \hspace{0.01\textwidth}
    \begin{subfigure}{0.235\textwidth}
        \centering
        \includegraphics[width=\linewidth]{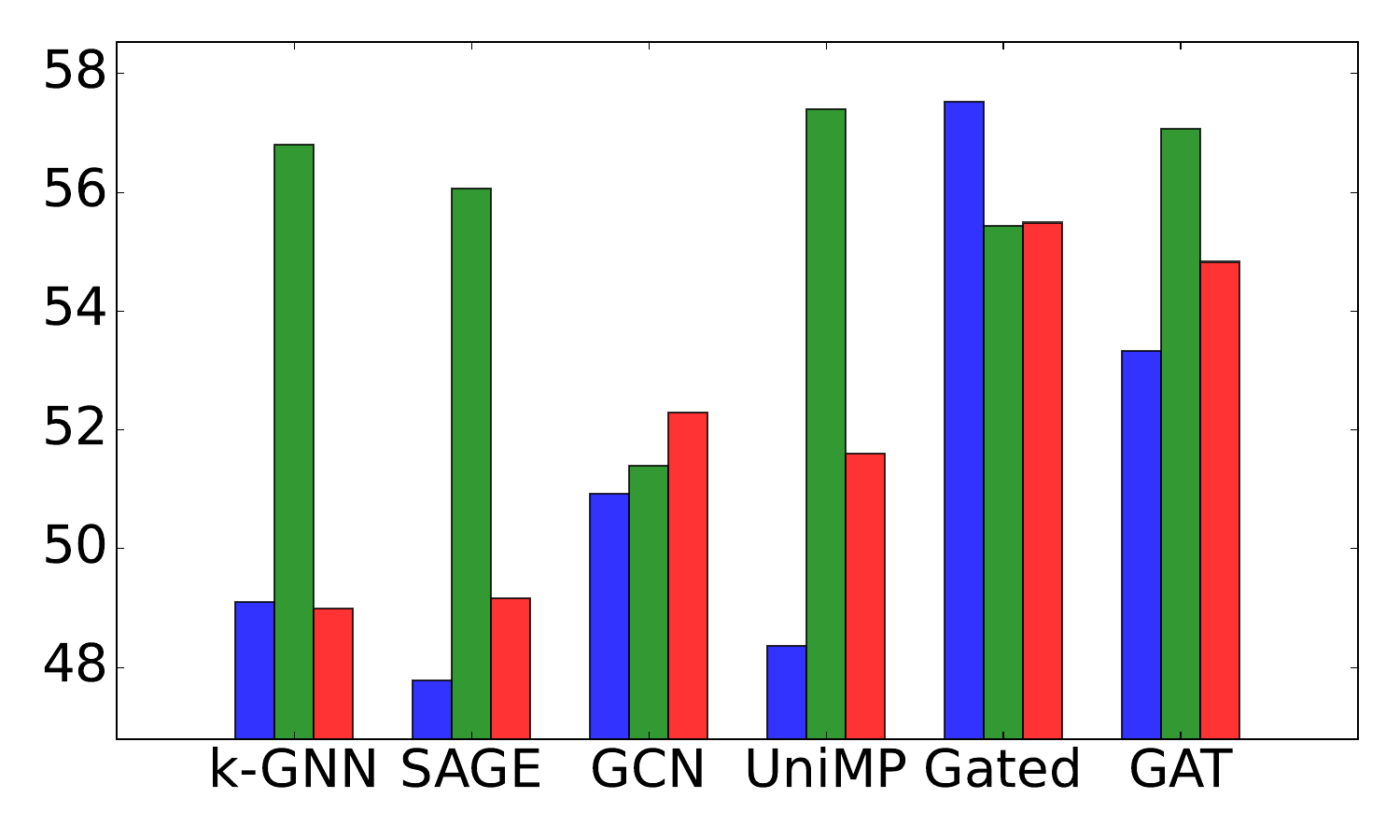}
        \caption{PTC}
    \end{subfigure}%
    \hspace{0.01\textwidth}
    \begin{subfigure}{0.235\textwidth}
        \centering
        \includegraphics[width=\linewidth]{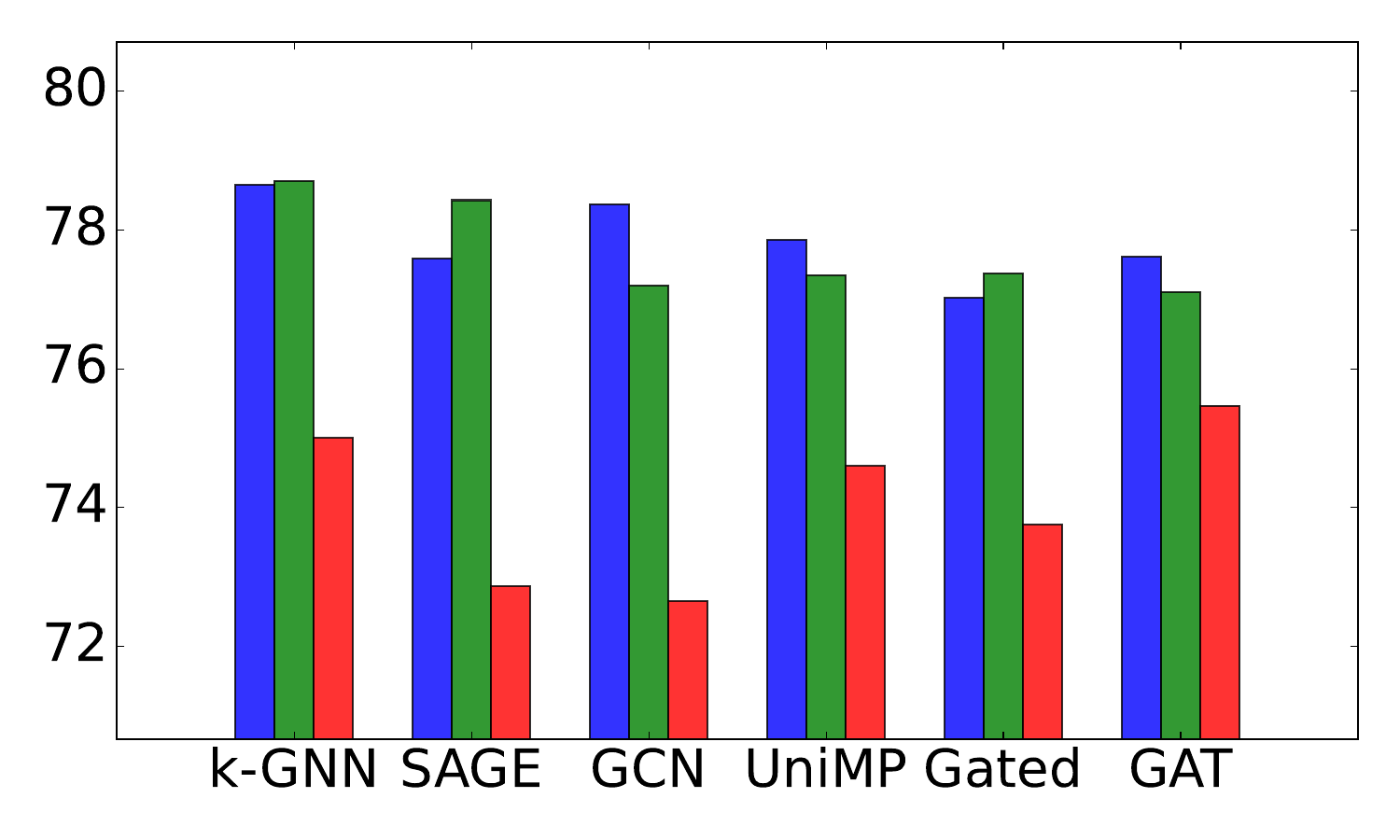}
        \caption{PROTEINS}
    \end{subfigure}%
    \hspace{0.01\textwidth}
    \begin{subfigure}{0.235\textwidth}
        \centering
        \includegraphics[width=\linewidth]{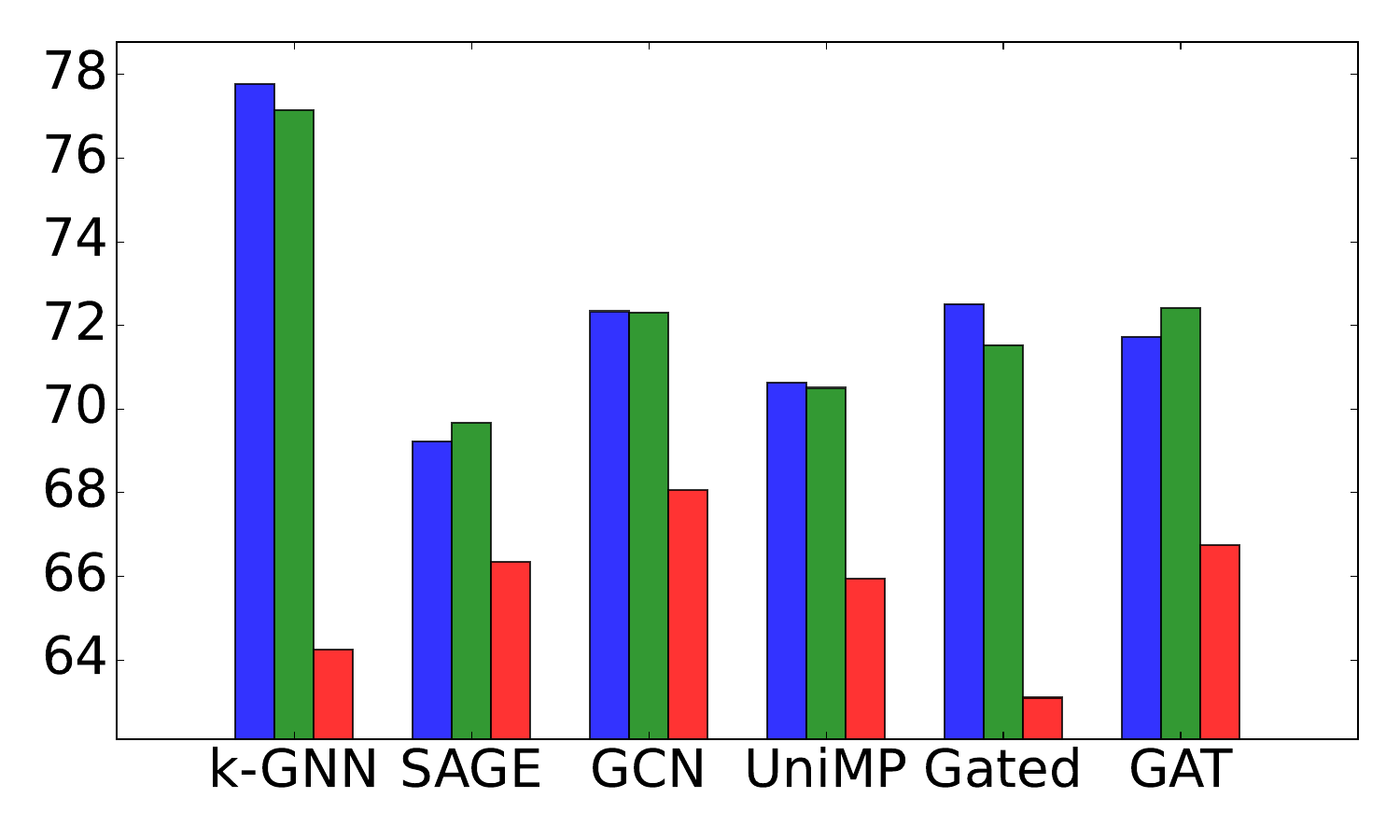}
        \caption{NCI1}
    \end{subfigure}
\begin{subfigure}{0.5\textwidth}
\vspace{-1mm}
\centering
  \includegraphics[width = \linewidth]{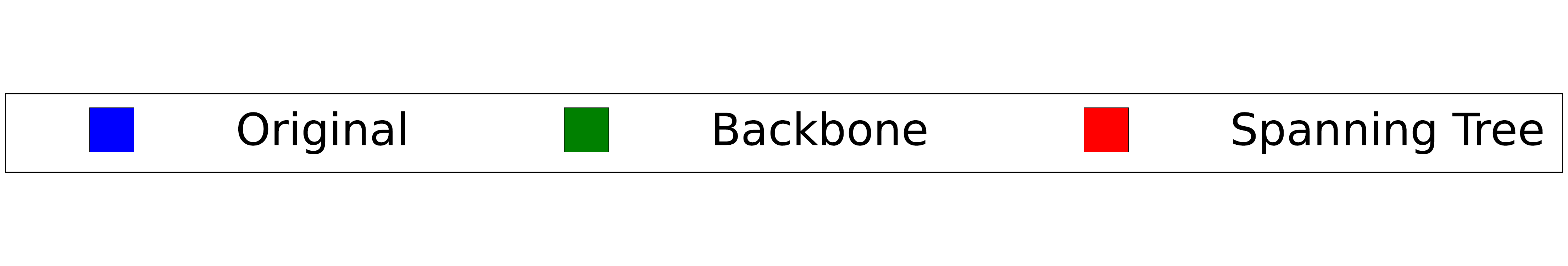}
\end{subfigure}%
\vspace{-2mm}
\caption{Comparing the Efficacy of Network Backbone Structures for Graph Classification. The backbone represents the best-performing structure between $B_z$ and $B_d$. The results are compared against the original graphs and random spanning tree subgraphs of the original graphs.}
\label{fig:results}
\end{figure*}

\section{Conclusion and Future Work}
\label{sec:conclusion}

This work develops an effective sparse machine learning backbone for graphs using a ZFS-based approach. This method simplifies graph structures into tree-like forms while retaining essential control properties, enhancing learning efficiency. Extensive experiments demonstrate that the ZFS-based backbone not only preserves network controllability but often outperforms original graphs and other sparse representations in graph classification tasks. We also explored a distance-based backbone, showing its potential to generalize the controllability backbone and preserve critical characteristics across diverse networks. The ZFS-based backbone provides a robust, efficient solution for improving graph learning by simplifying structures without sacrificing control attributes. Future research will refine the backbone computation process for large-scale applications and investigate the relationship between controllability and learning through the average degree of learning backbones.

\section{Acknolwdgements}
This work is supported by the National Science Foundation under Grant Numbers 2325416 and 2325417.

%
%
%
\bibliographystyle{spmpsci}
\bibliography{references}


\end{document}